\PassOptionsToPackage{dvipsnames}{xcolor}
\documentclass[accepted]{uai2026} 

\usepackage[british]{babel}
\usepackage{natbib}
\usepackage{mathtools}
\usepackage{booktabs}
\usepackage{graphicx}
\usepackage{tikz}
\usetikzlibrary{arrows.meta,positioning}
\graphicspath{{../fig/}}
\usepackage{amsfonts}
\usepackage{amssymb}
\usepackage{amsthm}
\usepackage{xcolor}
\usepackage{url}

\title{Operationalising Relative Causal Knowledge: \\ Backbone Identifiability from Private Reports on a Shared Outcome}

\author[1]{Fabrizio~Russo}
\author[2]{Mark~Somers}
\affil[1]{Department of Computing, Imperial College London\\}
\affil[2]{%
    Fifty One Degrees Ltd\\
    
    London, UK\\

    \href{mailto:fabrizio@imperial.ac.uk}{fabrizio@imperial.ac.uk}\\
    \href{mailto:mark.somers@outlook.com}{mark.somers@outlook.com}
}

\newtheorem{theorem}{Theorem}
\newtheorem{proposition}[theorem]{Proposition}
\newtheorem*{theorem*}{Theorem}
\newtheorem*{proposition*}{Proposition}
\theoremstyle{definition}
\newtheorem{example}{Example}
\newtheorem{remark}[theorem]{Remark}

\newcommand{\E}{\mathbb{E}}
\newcommand{\Cov}{\operatorname{Cov}}
\newcommand{\Var}{\mathrm{Var}}

\newcommand{\indep}{\mathrel{\perp\!\!\!\perp}}

\begin{document}

\maketitle

\begin{abstract}
The \emph{Relativity of Causal Knowledge} (RCK) explains how a network of agents with different structural causal models can exchange causal knowledge through a shared interventionally consistent abstraction, or backbone. We ask the prior identification question that this transport mechanism presupposes: when is that backbone determined by the agents' private causal knowledge? In the basic two-agent common-effect case, two private causes influence one shared outcome and each agent identifies only the single-cause causal marginal relevant to its own perspective. We show that, under standard compatibility, non-degeneracy, and local overlap assumptions, those local causal marginals do not identify a unique backbone. Infinitely many joint intervention kernels can induce exactly the same private reports while disagreeing on joint interventions. We then give a conditional recovery result. Additive separability removes the hidden interaction degree of freedom, but observational residual summaries remain insufficient. Identification becomes possible when agents communicate causally identified response functions. An education value-added example illustrates why this is first a communication problem, and only then a policy-composition problem.
\end{abstract}

\section{Introduction}
Causal knowledge is useful partly because it travels. A randomised trial, a quasi-experimental design, or a credible structural model does not merely answer one isolated question; it can become evidence for later decisions, related domains, or other researchers' models. Standard structural causal models (SCMs) provide the language of interventions and counterfactuals within one model~\citep{pearl2009}. 

The \emph{Relativity of Causal Knowledge} (RCK)~\citep{dacunto2025} pushes this idea into a distributed setting: different agents may hold different local SCMs for the same world, and causal knowledge can be transported between them if their local perspectives admit a shared, interventionally consistent abstraction, called a \emph{backbone}.

That conditional is powerful, but it leaves an operational question open. In applications, the backbone is usually not available to the agents. It must be inferred from the knowledge they privately possess. If that inference is underdetermined, two agents can each report a valid causal finding and still fail to identify the shared abstraction through which those findings can be shared accurately. RCK is mathematically defined \emph{given} a backbone, but the choice of backbone becomes an identification problem.

\begin{figure}[t]
\centering
\includegraphics[width=\linewidth]{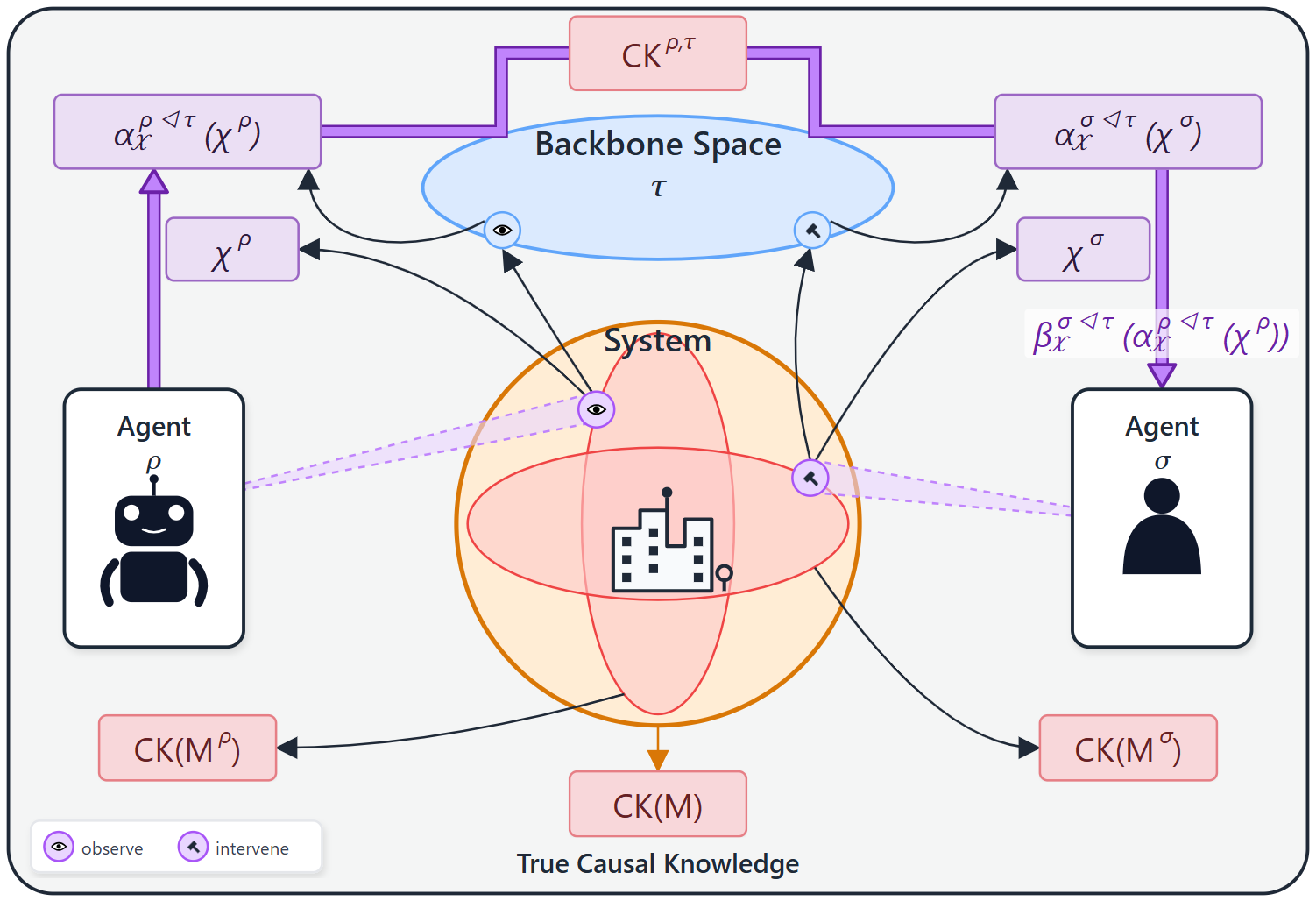}
\caption{RCK transport assumes a shared edge value in the backbone space $\tau$. The purple route shows local knowledge being restricted to the edge, aligned there, and extended into another agent's perspective; this paper asks when private reports identify that edge value in the first place.}\vspace{-0.5cm}
\label{fig:rck-transport}
\end{figure}
Figure~\ref{fig:rck-transport} separates the transport operation from the identification question studied here. Agent $\rho$ and Agent $\sigma$ are nodes in a network that can probe the same system, by intervention or just observation, from different local perspectives, producing causal-knowledge values $\chi^\rho\in CK(M^\rho)$ and $\chi^\sigma\in CK(M^\sigma)$. $CK(M)$ denotes the causal-knowledge object associated with a SCM $M$, which lives on the edge stalks of the network. If a common edge value $\chi^\tau\in CK(M^\tau)$ is available, RCK can restrict local knowledge to the backbone $\tau$ through \emph{restriction maps} ($\alpha_{\mathcal X}^{\rho\triangleleft\tau}$ and $\alpha_{\mathcal X}^{\sigma\triangleleft\tau}$), align it there, and extend it into another perspective through \emph{extension maps} ($\beta_{\mathcal X}^{\sigma\triangleleft\tau}$); more details in Appendix~\ref{app:rck-background}. Our question is the prior one: do the agents' private reports determine the backbone value $\chi^\tau$ that this restriction--alignment--extension pipeline needs? The following example instantiates the diagram with two local findings about the same outcome.

\begin{example}[Education value-added]\label{ex:education-motivation}
Consider a policy-maker asking whether better teachers can compensate for childhood neighbourhood disadvantage, e.g. in the context of the pupil premium funds in UK.\footnote{\url{https://www.gov.uk/government/publications/pupil-premium/pupil-premium}} Adult earnings $X$ are a natural shared outcome: the teacher value-added literature estimates long-run earnings effects of teacher quality~\citep{chetty2014teachers1,chetty2014teachers2}, while the neighbourhood-mobility literature estimates long-run earnings effects of childhood environment~\citep{chetty2016mto,chetty2018neighborhoods2}. The causal structure and effects of these findings is depicted in Figure~\ref{fig:education-common-effect}. Let Agent $\rho$ study teacher quality $Z_\rho$, and Agent $\sigma$ study neighbourhood environment $Z_\sigma$. Empirically, one year with a teacher one standard deviation higher in value-added is associated with roughly $1.3\%$ higher adult earnings, while one childhood year in a one-standard-deviation better county is associated with roughly $0.5\%$ higher adult earnings~\citep{chetty2014teachers2,chetty2018neighborhoods2}.

\begin{figure}[t]
\centering
\begin{tikzpicture}[>=stealth,scale=0.85,transform shape,font=\small]
\node[draw,rounded corners,fill=blue!8,minimum width=2.6cm,minimum height=0.85cm] (za) at (0,0) {teacher quality $Z_{\rho}$};
\node[draw,rounded corners,fill=green!8,minimum width=2.8cm,minimum height=0.85cm] (zb) at (4.2,0) {neighbourhood $Z_{\sigma}$};
\node[draw,rounded corners,fill=gray!10,minimum width=1.7cm,minimum height=0.85cm] (x) at (2.1,-1.95) {earnings $X$};
\draw[->,very thick,blue!70!black] (za) -- node[left=1pt,font=\scriptsize,align=center,text=blue!70!black] {\\$+1.3\%$ per year\\(value-added; Chetty et al., 2014a,b)} (x);
\draw[->,very thick,green!50!black] (zb) -- node[right=1pt,font=\scriptsize,align=center,text=green!40!black] {\\\\$+0.5\%$ per year\\(mobility; Chetty et al., 2016; \\Chetty and Hendren, 2018)} (x);
\node[draw=blue!60!black,rounded corners,fill=blue!4,font=\scriptsize,inner sep=2pt] at (-1.25,-0.45) {$\mathrm{do}(Z_{\rho}=z_{\rho})$};
\node[draw=green!40!black,rounded corners,fill=green!4,font=\scriptsize,inner sep=2pt] at (5.45,-0.45) {$\mathrm{do}(Z_{\sigma}=z_{\sigma})$};
\end{tikzpicture}
\caption{Two local causal findings about a shared outcome.}
\vspace{-0.3cm}
\label{fig:education-common-effect}
\end{figure}
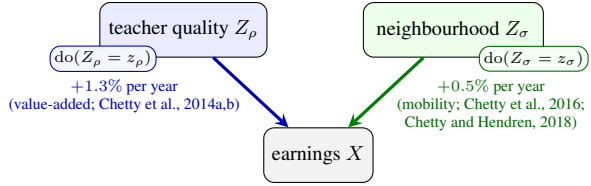

Both local findings can be correct. In the notation of Figure~\ref{fig:rck-transport}, they are local pieces of $\chi^\rho$ and $\chi^\sigma$: the teacher study reports how $X$ changes under interventions on $Z_\rho$, and the neighbourhood study reports the analogous response for $Z_\sigma$. What they do not determine is the edge value $\chi^\tau$ in the backbone. 
The RCK consequence is that communication itself becomes ambiguous. To use Agent $\rho$'s teacher finding inside Agent $\sigma$'s perspective, the agents must know which backbone value it is being extended through; but several candidate $\chi^\tau$ values can agree with both local findings while disagreeing on a joint intervention $\mathrm{do}(Z_\rho=z_\rho,Z_\sigma=z_\sigma)$. Thus Agent $\rho$'s causal knowledge has no unique interventionally meaningful translation for Agent $\sigma$, and any compensatory policy calculation inherits that arbitrary choice.
\end{example}

Our contribution is to make this missing communication step precise. First, we formulate backbone recovery as an edge-level identification problem in the language of RCK. Second, for common-effect backbones, we prove a kernel-level non-identifiability theorem: under explicit non-degeneracy and local minorisation conditions, local causal marginals can leave infinitely many causally distinct backbones compatible with the same private reports, so the agents do not know which edge value they are aligning on. Third, we show how the obstruction can be removed under additive separability, provided that agents communicate causally identified response functions rather than observational summaries.
\vspace{-0.4cm}
\paragraph{Position in the Literature.}
The closest point of departure is RCK itself~\citep{dacunto2025}. RCK studies how causal knowledge can be transported across a \emph{network sheaf and cosheaf of causal knowledge} when local SCMs admit a common interventionally consistent abstraction. Related causal-abstraction work studies exact transformations, semantic embedding, and compositionality across levels of description~\citep{rubenstein2017consistency,beckers2019abstraction,rischel2021compositional,massidda2023soft,dacunto2025semantic}. We do not ask whether transport is well behaved once an abstraction is available, nor whether an abstraction can be learned between  specified models. Our question is earlier, about the target of causal abstractions.

The problem is also distinct from aggregating expert causal judgements outside the abstraction setting. \citet{bradley2014aggregating} study aggregation rules for causal-network judgements and associated probabilities over a common variable set. \citet{alrajeh2018experts} instead take experts' causal models as inputs and define compatibility and dominance conditions under which those models can be merged, while \citet{friedenberg2018focus} extend this programme to experts with different focus areas using a \emph{can-explain} relation. Our inputs and target are different. We do not assume access to full expert models, structural equations, or focus sets, and we are not trying to select a collective graph or merged SCM. We ask whether local causal reports determine the RCK backbone required by abstraction maps; the goal is the backbone identifiability rather than graph or model aggregation.

Finally, transportability, data fusion, and federated causal inference study how evidence from different environments or sources can be combined under explicit assumptions~\citep{pearl2014transportability,bareinboim2016datafusion,xiong2023federated,li2024fedcdh}. Causal marginal and latent-variable compatibility work similarly warns that partial causal views need not determine a unique joint source~\citep{tian2002testable,gresele2022}. We share that intuition, but the target differs: the unidentified object is not a pooled treatment effect, a merged graph, or a full joint SCM. It is the edge in the network sheaf and cosheaf through which RCK would allow sharing abstracted causal knowledge. 

We provide additional discussion and connections to the literature in Appendix~\ref{app:expanded-related}.
\section{Backbone Identification from Colliding Private Reports}
\vspace{-0.2cm}
Recall the two-agent scenario of Example~\ref{ex:education-motivation}. Agent $\rho$ has a private cause $Z_\rho$, Agent $\sigma$ has a private cause $Z_\sigma$, and both care about the same outcome $X$. This creates a common-effect, collider configuration $Z_\rho\to X\leftarrow Z_\sigma$: the backbone answers to joint interventions on both private causes. In this setting a mixed response coordinate can be hidden from both one-cause reports. 
In RCK, the backbone is the shared edge object on which local causal knowledge must agree before it can be transported. In this two-agent setting, that edge value is represented concretely by the interventionally consistent kernel
\begin{equation*}\label{eq:kernel}
Q(H\mid z_\rho,z_\sigma)
=
P\!\left(X\in H\mid \mathrm{do}(Z_\rho=z_\rho,Z_\sigma=z_\sigma)\right),\\[-0.2cm]
\end{equation*}
for measurable outcome events $H$. This joint-intervention kernel represents the edge value $\chi^\tau$. With $\chi^\tau$ fixed, restriction and extension maps have a common object through which to transport abstracted causal knowledge. 
%
If the other private cause is not part of the agent's local design, a candidate backbone $Q$ projects to single-cause reports by averaging over that hidden private cause:
\vspace{-0.3cm}
\begin{equation}\label{eq:private-reports}
\begin{aligned}
P_\rho(\cdot\mid z_\rho)
&:=
\int Q(\cdot\mid z_\rho,z_\sigma)\,P_{Z_\sigma}(dz_\sigma),\\
P_\sigma(\cdot\mid z_\sigma)
&:=
\int Q(\cdot\mid z_\rho,z_\sigma)\,P_{Z_\rho}(dz_\rho).
\end{aligned}
\end{equation}
These 
\emph{local report operators}
describe how private information hides the other agent's knowledge. The identification problem fixes the reported kernels $(P_\rho,P_\sigma)$ and asks whether they determine a unique, interventionally consistent $Q$. Thus Eq.~\eqref{eq:private-reports} should be read as a report-projection constraint: a candidate backbone $\bar Q$ is compatible if
\vspace{-0.3cm}
\[
\begin{aligned}
\int \bar Q(\cdot\mid z_\rho,z_\sigma)\,P_{Z_\sigma}(dz_\sigma)
&=P_\rho(\cdot\mid z_\rho),\\
\int \bar Q(\cdot\mid z_\rho,z_\sigma)\,P_{Z_\rho}(dz_\rho)
&=P_\sigma(\cdot\mid z_\sigma).
\end{aligned}
\]
Backbone identification requires all compatible candidates to agree on the intervention queries of interest. If two compatible kernels agree with both private reports but disagree on a joint intervention, then the edge value $\chi^\tau$ needed for RCK transport is not identified; Appendix~\ref{app:assumptions} gives the corresponding compatibility set and scope.

\vspace{0.1cm}
\noindent\textbf{Generic Non-Identifiability.}
We prove that such compatible kernels are generically not unique under weak assumptions: First, the two reports must be compatible with at least one candidate backbone $Q_0$; Second, each private-cause space must contain a nontrivial centred direction: bounded nonzero functions $u(Z_\rho)$ and $v(Z_\sigma)$ with mean zero under the private-cause laws $P_{Z_\rho}$ and $P_{Z_\sigma}$; Third, $Q_0$ must satisfy a local overlap requirement, formulated as a minorisation or small-set condition~\citep[Sec.~5.1]{meynTweedie1993}, so that a small signed perturbation can be added without leaving the space of probability kernels; see Appendix~\ref{app:assumptions}.

\begin{theorem}[Backbone non-identifiability under private knowledge]\label{thm:main}
Assume:
\begin{enumerate}
\item\label{ass:A1} there exists at least one compatible backbone kernel $Q_0$;
\item\label{ass:A2} there are measurable sets $A_\rho$ and $A_\sigma$ with positive private-cause probability, and bounded nonzero mean-zero functions $u(Z_\rho)$ and $v(Z_\sigma)$ supported on $A_\rho$ and $A_\sigma$, respectively;
\item\label{ass:A3} for some probability measure $\nu$ with $0<\nu(B)<1$ for some measurable $B$, and some $\varepsilon>0$,
\(
Q_0(H\mid z_\rho,z_\sigma)\geq \varepsilon\,\nu(H)
\)
for every measurable event $H$ and all $(z_\rho,z_\sigma)\in A_\rho\times A_\sigma$.
\end{enumerate}
Then there is a radius $\delta>0$ and a one-parameter family of distinct kernels $\{Q_t:t\in(-\delta,\delta)\}$ such that every $Q_t$ induces the same private reports $P_\rho$ and $P_\sigma$, but for $t\neq s$ the kernels $Q_t$ and $Q_s$ disagree on the joint interventional distribution for some intervention pair $(z_\rho,z_\sigma)$.
\end{theorem}

\begin{proof}[Proof sketch]
Start from one compatible kernel $Q_0$. The proof has three steps. First, choose a nontrivial signed perturbation $S$ on the outcome space with total mass zero, so adding it can change some events without breaking normalisation. Second, multiply this perturbation by centred functions $u(z_\rho)$ and $v(z_\sigma)$ on the two private-cause spaces. This makes the perturbation disappear whenever either agent averages over the other agent's private cause. Third, use the local minorisation condition to keep the perturbed laws non-negative for sufficiently small amplitudes. Define
\vspace{-0.1cm}
\begin{equation}\label{eq:perturbation}
Q_t(\cdot\mid z_\rho,z_\sigma)
=
Q_0(\cdot\mid z_\rho,z_\sigma)
+t\,u(z_\rho)v(z_\sigma)S(\cdot).\\[-0.1cm]
\end{equation}
The product $u(z_\rho)v(z_\sigma)$ is the hidden interaction direction: it is invisible to both one-cause reports because each report integrates out one centred factor, e.g.
\(
\int u(z_\rho)v(z_\sigma)P_{Z_\sigma}(dz_\sigma)=0
\)
and symmetrically for $Z_\rho$. But the same product is visible before marginalisation, and because $u$, $v$, and $S$ are nontrivial, two different amplitudes $t\neq s$ disagree for some joint intervention pair. Thus the same private reports are compatible with infinitely many causally distinct backbone candidates.
\end{proof}

All formal proofs are collected in Appendix~\ref{app:formal-proofs}.

In RCK terms, the failure occurs at the edge stalk level. The agents' private reports can agree with many candidate edge values (represented by kernels $Q_t$), so there is no unique $\chi^\tau$ through which restriction, alignment, and extension should proceed (more details in Appendix~\ref{app:rck-proof}). The following linear witness makes the communication failure concrete.
\begin{example}[Why naive communication fails in the education case]\label{ex:education-impossible}
A minimal witness is
\vspace{-0.1cm}
\[
X=\alpha Z_\rho+\beta Z_\sigma+\gamma Z_\rho Z_\sigma+\varepsilon,
\qquad
\E[Z_\rho]=\E[Z_\sigma]=0.\\[-0.1cm]
\]
The single-cause intervention responses identify $\alpha$ and $\beta$:
\(
\E[X\mid \mathrm{do}(Z_\rho=z_\rho)] = \alpha z_\rho,
\E[X\mid \mathrm{do}(Z_\sigma=z_\sigma)] = \beta z_\sigma
\), respectively.
They do not identify $\gamma$. Thus the same teacher and neighbourhood reports are compatible with compensatory effects ($\gamma<0$), additive effects ($\gamma=0$), or complementary effects ($\gamma>0$). The agents can exchange truthful local summaries, but those summaries do not specify a unique backbone value and therefore do not support unambiguous RCK communication; more details and other example witnesses are in Appendix~\ref{app:witnesses}. In policy terms, this leaves open whether better teachers help most in disadvantaged neighbourhoods, help equally everywhere, or are amplified by neighbourhood advantage.
\end{example}

\paragraph{A Conditional Recovery Result.}
The perturbation in Eq.~\eqref{eq:perturbation} is invisible locally because it is a mixed term in the two private causes. A natural way to rule out that ambiguity is additive separability:
\vspace{-0.2cm}
\begin{equation}\label{eq:separable}
X=f_\rho(Z_\rho)+f_\sigma(Z_\sigma)+\varepsilon.\\[-0.2cm]
\end{equation}
This assumption should not be read as a free modelling convenience. In the education example, it is motivated by the way the two literatures are usually parameterised and by the institutional separation between classroom instruction and neighbourhood or family environment~\citep{hanushek1979,rockoff2004impact,rivkin2005teachers}. It is also empirically revisable; interactions can exist and should be tested where joint data are available~\citep{jackson2014tracks}. The role of separability here is precise: it removes the hidden interaction coordinate that drives non-identifiability in Theorem~\ref{thm:main}.

\begin{proposition}[Additive separability removes the mixed response]\label{prop:separable}
If the shared outcome $X$ satisfies Eq.~\eqref{eq:separable} and $\E[\varepsilon]<\infty$, then the joint interventional mean response is
\vspace{-0.1cm}
\[
\begin{aligned}
m(z_\rho,z_\sigma)
&:=
\E[X\mid \mathrm{do}(Z_\rho=z_\rho,Z_\sigma=z_\sigma)]\\
&= 
f_\rho(z_\rho)+f_\sigma(z_\sigma)+\E[\varepsilon].\\[-0.1cm]
\end{aligned}
\]
In particular, the response contains no mixed term in $(z_\rho,z_\sigma)$.
\end{proposition}

Separability is the structural half of recovery, but it does not imply identifiability. Even with the structural assumption of Eq.~\eqref{eq:separable}, arbitrary observational summaries do not identify the additive backbone. For instance, after subtracting its own causal channel, Agent $\rho$ and Agent $\sigma$ have
\vspace{-0.1cm}
\[
\begin{aligned}
R_\rho&=X-f_\rho(Z_\rho)=f_\sigma(Z_\sigma)+\varepsilon,\\
R_\sigma&=X-f_\sigma(Z_\sigma)=f_\rho(Z_\rho)+\varepsilon.
\end{aligned}
\]
These are agent-relative residuals, not common backbone residuals: $R_\rho$ still bundles the other private response $f_\sigma(Z_\sigma)$ with the disturbance $\varepsilon$, while $R_\sigma$ bundles $f_\rho(Z_\rho)$ with $\varepsilon$. A residual variance alone therefore does not separate variation due to the other causal channel from variation due to $\varepsilon$; Appendix~\ref{app:obs-proof} provides a formal argument of this. 
\begin{remark}
    What resolves the ambiguity is causal communication. Once Agent $\rho$ communicates $f_\rho$ and Agent $\sigma$ communicates $f_\sigma$, the backbone mean response \(m(z_\rho,z_\sigma)=f_\rho(z_\rho)+f_\sigma(z_\sigma)\) is fixed for every joint intervention pair, up to the shared constant $\E[\varepsilon]$ when absolute outcome levels rather than contrasts are required.
\end{remark}
More detail and two specific cases are in Appendix~\ref{app:positive-special-cases}.

\begin{example}[Education value-added under causal communication]\label{ex:education-recovery}
Following~\citep{chetty2014teachers2,chetty2018neighborhoods2}, we specialise the additive response to a linear case and measure exposures in standard-deviation-year units:
\(
f_\rho(z_\rho)=1.3\,z_\rho,
f_\sigma(z_\sigma)=0.5\,z_\sigma,
\)
with effects in percentage points. A full-childhood neighbourhood disadvantage of roughly one standard deviation, assuming a full childhood comprises eighteen one-year exposure periods and applying the per-year exposure estimate from \cite{chetty2018neighborhoods2} linearly over that exposure window, corresponds to about
\(
\Delta_\sigma := 18\times 0.5\% \approx 9\%
\)
lower adult earnings. If $\bar z_\rho$ is the baseline teacher-quality exposure and $z_\rho^\ast$ is the compensating exposure, the offset solves
\vspace{-0.1cm}
\[
1.3\,(z_\rho^\ast-\bar z_\rho)
=
f_\rho(z_\rho^\ast)-f_\rho(\bar z_\rho)
=
\Delta_\sigma.\\[-0.1cm]
\]
The required cumulative increase is therefore approximately $\Delta_\sigma/1.3\approx 9/1.3\approx 7$ teacher-quality standard-deviation years. If spread evenly over a $13$-year school career, this is about $7/13\approx0.54$ standard deviations above baseline per year.
Within this linear additive special case, the communicated slopes suffice to identify the backbone response; if $Z_\rho\indep Z_\sigma$, $\varepsilon\indep(Z_\rho,Z_\sigma)$, and $\Var(X)$ is observed, the residual variance is recovered by
\(
\Var(\varepsilon)=\Var(X)-\Var\!\bigl(f_\rho(Z_\rho)\bigr)-\Var\!\bigl(f_\sigma(Z_\sigma)\bigr).
\)
This calculation is not licensed by the local marginals alone. It becomes meaningful only after separability rules out the hidden interaction regime and the agents communicate causally identified response functions. Under those conditions, the neighbourhood study identifies the size of the disadvantage gap, the teacher study identifies the compensating instructional dose, and the RCK backbone becomes an operational object enabling causal inference from private reports.
\end{example}
\vspace{-0.4cm}
\section{Conclusion}
\vspace{-0.2cm}
RCK is a theory of causal transport given a shared \emph{interventionally consistent} abstraction. This paper studies the identification problem that precedes such transport in private-knowledge common-effect settings. The main message is negative generically and positive conditionally. Private local causal marginals need not determine a unique edge-level backbone, and the ambiguity is causally meaningful because compatible backbones can disagree on joint interventions. Additive separability removes the interaction ambiguity, but agents must still communicate causal response objects rather than observational residual summaries.

This gives a narrower but actionable reading of RCK. The framework's transport maps are useful once the shared edge value is available. Our results clarify when private causal knowledge can supply that edge value, and when additional structure or communication is required. Natural next steps are to test separability in pooled studies, extend the result to larger agent networks, and identify weaker equivalence classes of backbones that preserve only the intervention queries needed for a given decision.

\acknowledgements
Russo was supported by the ERC under the ERC-POC programme (grant number 101189053).

\bibliography{../refs}

\newpage
\appendix

\section{SCM and RCK Background}\label{app:rck-background}
\vspace{-0.2cm}
An SCM specifies endogenous variables $V=(X_1,\dots,X_d)$, exogenous variables $U=(U_1,\dots,U_d)$, a directed acyclic graph $G$ over $V$, a joint distribution $P_U$, and structural assignments
\[
X_j := f_j(X_{\mathrm{Pa}(j)},U_j),
\qquad j=1,\dots,d.
\]
Interventions replace the structural assignments of intervened variables by constants, yielding interventional distributions such as $P(X\in H\mid \mathrm{do}(Z=z))$~\citep{pearl2009,peters2017elements}. Throughout, state spaces are measurable spaces and all events such as $H\subseteq\mathcal X$ are measurable.

Figure~\ref{fig:rck-transport} gives a compact view of the RCK objects used in the paper. Start with a network of agents $N=(\mathcal N,\mathcal E)$. Nodes $\rho,\sigma\in\mathcal N$ represent local perspectives on the same world: each agent has a local SCM and hence a local causal-knowledge space, written $CK(M^\rho)$ or $CK(M^\sigma)$. A node value such as $\chi^\rho$ is the agent's local section of causal knowledge. It may include observational facts, interventional responses, and any other causal quantities represented by the local SCM; these are local because they are generated from that agent's perspective, measurements, interventions, and modelling assumptions. Edges $\tau\in\mathcal E$ represent shared abstractions, or \emph{backbones}, through which adjacent agents may communicate. Thus the lower layer of the picture is the agent network, while the boxes attached to nodes and edges are the causal-knowledge spaces that sit on top of it.

RCK organises these spaces as a network sheaf and cosheaf of causal knowledge. The sheaf assigns a \emph{stalk} to each node and edge and assigns, to every node--edge incidence $\rho\triangleleft\tau$, a restriction map
\(
\alpha_{\mathcal X}^{\rho\triangleleft\tau}:CK(M^\rho)\to CK(M^\tau).
\)
In the figure, this is the step that takes Agent $\rho$'s local causal value $\chi^\rho$ and projects the part relevant to the shared outcome space $\mathcal X$ into the backbone $\tau$. The dual cosheaf carries information in the opposite direction by extension maps
\(
\beta_{\mathcal X}^{\sigma\triangleleft\tau}:CK(M^\tau)\to CK(M^\sigma),
\)
shown by the arrow that transports backbone-level content into Agent $\sigma$'s perspective.

The formal alignment condition is that the node values agree after restriction to the edge:
\[
\alpha_{\mathcal X}^{\rho\triangleleft\tau}(\chi^\rho)
=
\alpha_{\mathcal X}^{\sigma\triangleleft\tau}(\chi^\sigma)
=
\chi^\tau.
\]
This is the one-edge form of the global-section condition in RCK~\citep[Def. 10]{dacunto2025}: a compatible family of local values must induce a common value on every shared edge. Once such an edge value $\chi^\tau$ is fixed, the transported causal knowledge from $\rho$ to $\sigma$ is
\[
\chi^{\rho,\sigma}
=
\beta_{\mathcal X}^{\sigma\triangleleft\tau}
\!\left(
\alpha_{\mathcal X}^{\rho\triangleleft\tau}(\chi^\rho)
\right).
\]
We study the step before this transport operation. If private reports do not identify a unique edge value $\chi^\tau$, then restriction and extension maps may remain formally well defined for each candidate backbone, but communication through the edge is operationally ambiguous.

\section{Expanded Related Work and Connections}\label{app:expanded-related}

\paragraph{RCK, Abstraction, and Compositionality.}
RCK~\citep{dacunto2025} studies how causal knowledge can be transported across a \emph{network sheaf and cosheaf of causal knowledge} built from interventionally consistent backbones. Closely related abstraction formalisms include exact transformation and consistency across causal models~\citep{rubenstein2017consistency,beckers2019abstraction}, approximate or composable abstractions~\citep{rischel2021compositional,massidda2023soft}, and the semantic embedding principle~\citep{dacunto2025semantic}. Our focus is upstream of these guarantees. Rather than assuming a shared edge value and studying how abstraction maps behave, we ask when the target of these abstractions itself is identifiable from private causal reports.

The main conceptual point is that the ambiguity we identify does not primarily live in the existence of the edge $\tau$ as a graph relation within the RCK network, nor in the existence of the RCK morphisms as maps. It lives in the failure of private local causal knowledge to identify a unique edge value $\chi^\tau$ on the edge stalk. Extension then propagates that ambiguity rather than creating it. This distinguishes our setting from many abstraction frameworks, where the primary ambiguity lies in selecting or learning an abstraction map between already-given causal models. Here the maps are fixed once an edge object is fixed; what is missing is the edge object itself. This also points to a natural empirical escape hatch: if a later study or consortium can collect joint data on $(Z_\rho,Z_\sigma,X)$, then the additivity assumption from Proposition~\ref{prop:separable} becomes at least partly testable because an interaction term can be estimated or falsified directly rather than inferred from separate local reports.

\paragraph{Study Combination and Marginal Compatibility.}
Transportability and data-fusion theories ask when causal information from different environments or studies can be combined~\citep{pearl2014transportability,bareinboim2016datafusion}; related work studies causal invariance across environments~\citep{rojas2018invariant}. Federated causal inference and federated causal discovery ask analogous questions under privacy or decentralisation constraints~\citep{xiong2023federated,li2024fedcdh}. Causal marginal and latent-variable compatibility work similarly warns that partial causal views need not determine a unique joint source~\citep{tian2002testable,gresele2022}. The binary witness of \S\ref{app:witness-binary} is a causal version of the familiar fact that fixed row and column sums leave one interaction coordinate free in a $2\times 2$ table~\citep{agresti2013categorical}. More broadly, compatible marginals need not determine a unique joint distribution~\citep{nelsen2006copulas}. The structural causal marginal problem studies how much of a joint SCM can be recovered from overlapping causal marginals~\citep{gresele2022}. Our object is narrower: the unidentified object is the RCK edge value through which local causal knowledge would be transported.

\paragraph{Expert Judgement and Decision-Making.}
The problem is distinct from aggregating expert causal judgements. \citet{bradley2014aggregating} study aggregation rules for causal-network judgements and associated probabilities over a common variable set. \citet{alrajeh2018experts} take experts' SCMs as inputs and define compatibility and dominance conditions for merging, while \citet{friedenberg2018focus} extend this programme to experts with different focus areas. Our inputs are private causal reports, not full expert models. The decision-making examples in the main text show why this distinction matters: a decision-maker can need a joint response even when all available local causal findings are individually credible.

\paragraph{Beyond Standard Observational Equivalence.}
Linear Gaussian SCMs are known to be non-identifiable from observational data~\citep{peters2017elements}, but the standard language in that setting is Markov equivalence or edge-orientation ambiguity. Our ambiguity is different. The unresolved degree of freedom sits in the candidate backbone itself and can alter interventional predictions while leaving both local causal marginals unchanged. In that sense, the relevant equivalence class in our setting is not a standard Markov equivalence class but a coarser class where independence constraints are not consistent across different backbones.

\paragraph{ICA and Non-Gaussian Identification.}
Classical ICA shows that a linear mixture of independent non-Gaussian components is identifiable up to scaling and permutation~\citep{comon1994}. LiNGAM imports that idea into causal discovery: in a linear acyclic SEM with independent non-Gaussian disturbances and no latent confounding, the full causal structure becomes identifiable from observational data because the mixing matrix can be recovered and then aligned with a causal ordering~\citep{shimizu2006lingam}. The broader functional-model literature extends this logic beyond the linear non-Gaussian case to additive-noise and identifiable functional model classes~\citep{hoyer2008anm,peters2011ifmoc}. Our positive result is different in both target and mechanism. We do not identify a full DAG or mixing matrix from a single joint law, and we do not use non-Gaussianity of disturbances as the identifying lever. Instead, the object to be identified is an edge-level backbone, and identification comes from additive separability at the shared outcome together with communicated causally identified responses.

\paragraph{Abstraction, Semantic Embedding, and SEP.}
The abstraction literature provides the closest motivational parallel. In work on exact and approximate causal abstraction, the goal is to relate causal models across levels of description~\citep{beckers2019abstraction,rischel2021compositional,massidda2023soft}. The semantic embedding principle (SEP)~\citep{dacunto2025semantic} is especially relevant as a contrast case: SEP-style abstraction learning asks which abstraction map $\alpha_{\mathcal X}$ pushes a low-level model onto a higher-level one and admits a measurable right inverse. In such settings, non-identifiability often appears at the level of the \emph{abstraction morphism}: multiple maps may be compatible with the same observed low- and high-level distributions.

Our theorem isolates a complementary problem. The local-to-edge semantics are fixed by the problem: the agents observe only marginals of a shared but hidden edge-level object. Thus the primary ambiguity is not whether restriction or extension can be defined as maps, but whether private local reports identify a unique edge value $\chi^\tau$ on which local agreement should occur. In the two-node RCK setting, local agreement is
\[
\alpha_{\mathcal X}^{\rho\triangleleft\tau}(\chi^\rho)
=
\alpha_{\mathcal X}^{\sigma\triangleleft\tau}(\chi^\sigma)
=
\chi^\tau.
\]
Our claim is upstream of this condition: there may be several admissible candidates for $\chi^\tau$ that are all compatible with the same private reports, and hence several candidate local agreements with no data-driven way to select a unique one. Once one such value is fixed, extension proceeds in the ordinary RCK way:
\(
\chi^{\rho,\sigma}
=
\beta_{\mathcal X}^{\sigma\triangleleft\tau}(\chi^\tau).
\)
SEP is therefore relevant as a comparison, not as the main formal home of our theorem.

This distinction can be stated cleanly in pushforward language. In semantic embedding, one seeks a map $\phi$ such that
\(
\mu_h=\phi_\#\mu_\ell.
\)
In our setting, the coordinate projections are fixed:
\[
\mu_\rho=(\pi^\rho)_\#\mu,
\qquad
\mu_\sigma=(\pi^\sigma)_\#\mu,
\]
but the pair of marginal pushforwards does not determine the source measure $\mu$. Thus both theories concern what can and cannot be recovered from pushforward-compatible information, but at different levels. One underdetermines the map; the other underdetermines the underlying joint object. The shared lesson is that abstraction or communication from compressed causal summaries generically produces equivalence classes rather than unique solutions unless additional structure is supplied.

\paragraph{Compositionality.}
This also connects to recent work that frames identification itself compositionally~\citep{friend2025compositional}. Our theorem can be read as a negative compositionality result: valid local causal pieces do not automatically compose into a unique global causal object. Proposition~\ref{prop:separable} identifies the structural regime in which the hidden interaction term disappears, and the communication condition in Section~3 records when the remaining additive object becomes recoverable. In this sense, additive separability is the identifying form of compositionality used in the paper. It is not merely that the two private mechanisms are both causal; they must combine without an unidentified interaction term if their communicated causal response objects are to support a unique backbone.

\section{Assumptions and Scope}\label{app:assumptions}

For fixed private reports, the compatible backbone candidates are
\begin{equation}\label{eq:compat-main}
\begin{aligned}
\mathcal B_\tau(P_\rho,P_\sigma)
&:=
\left\{
\chi^\tau\in CK(M^\tau):
\right.\\[-0.2em]
&\hspace{3em}\left.
\pi^\rho(\chi^\tau)=P_\rho,\;
\pi^\sigma(\chi^\tau)=P_\sigma
\right\}.
\end{aligned}
\end{equation}
Backbone identification asks whether this set collapses to a unique interventionally consistent edge object.

The assumptions in Theorem~\ref{thm:main} separate existence, non-degeneracy, and overlap. Assumption~\ref{ass:A1} is a baseline compatibility requirement: before asking whether the backbone is unique, the two private reports must admit at least one common backbone candidate. Without it, the problem is inconsistency rather than non-identifiability.

Assumption~\ref{ass:A2} requires nontrivial centred directions in both private-cause spaces. This is the source of the hidden interaction direction. If one private-cause space has no nontrivial variation, or if the local report already conditions on all relevant causes, then the specific common-effect obstruction studied here disappears. The condition is weak in ordinary statistical settings: any non-degenerate real-valued private cause admits bounded centred functions on positive-probability sets.

Assumption~\ref{ass:A3} is a local overlap, or small-set, condition~\citep[Sec.~5.1]{meynTweedie1993}. It ensures that a small signed perturbation can be added to a compatible kernel without making probabilities negative. The condition is local rather than global: it only needs to hold on a positive-measure rectangle $A_\rho\times A_\sigma$. This accommodates kernels that are nearly deterministic outside that region and common continuous kernels whenever they have positive mass over a shared outcome set.

The theorem is therefore not a claim that every RCK edge is unidentified. It targets common-effect backbones where two private causes jointly affect a shared outcome and each report marginalises over the other cause. Sequential settings can compose through ordinary kernel composition, and pure common-cause overlaps create different identification problems. The positive result is likewise scoped: additive separability identifies the mean response surface once the response functions are communicated; identifying the full distributional kernel also requires the disturbance law or an equivalent distributional object.

\section{Proofs}\label{app:formal-proofs}

This appendix gives the full proof of Theorem~\ref{thm:main}, the proof of Proposition~\ref{prop:separable}, an observational-summary counterexample, and the RCK interpretation of Theorem~\ref{thm:main}.

\subsection{Proof of Theorem~\ref{thm:main}}\label{app:nonid-proof}
\begin{theorem*}[Theorem~\ref{thm:main}, restated]
Assume:
\begin{enumerate}
\item there exists at least one compatible backbone kernel $Q_0$;
\item there are measurable sets $A_\rho$ and $A_\sigma$ with positive private-cause probability, and bounded nonzero mean-zero functions $u(Z_\rho)$ and $v(Z_\sigma)$ supported on $A_\rho$ and $A_\sigma$, respectively;
\item for some probability measure $\nu$ with $0<\nu(B)<1$ for some measurable $B$, and some $\varepsilon>0$,
\(
Q_0(H\mid z_\rho,z_\sigma)\geq \varepsilon\,\nu(H)
\)
for every measurable event $H$ and all $(z_\rho,z_\sigma)\in A_\rho\times A_\sigma$.
\end{enumerate}
Then there is a radius $\delta>0$ and a one-parameter family of distinct kernels $\{Q_t:t\in(-\delta,\delta)\}$ such that every $Q_t$ induces the same private reports $P_\rho$ and $P_\sigma$, but for $t\neq s$ the kernels $Q_t$ and $Q_s$ disagree on the joint interventional distribution for some intervention pair $(z_\rho,z_\sigma)$.
\end{theorem*}

The proof is organised around three tasks. First, we need a perturbation direction on the outcome space that changes some joint response but preserves total mass; that is the role of Proposition~\ref{prop:signed-perturbation-new}. Second, we need to lift that direction into a family of admissible backbone kernels that still reproduces the same private reports; that is the role of Proposition~\ref{prop:kernel-family-new}. Third, we need to show that different members of this family are not merely different parameterisations of the same object, but disagree on at least one joint intervention; that is the role of Proposition~\ref{prop:interventional-nonequivalence-new}. 
\begin{proposition}[Signed perturbation on the outcome space]\label{prop:signed-perturbation-new}
Let $\nu$ be a probability measure on $\mathcal X$. If there exists a measurable set $B\subseteq\mathcal X$ with $0<\nu(B)<1$, then there exists a bounded measurable function $s:\mathcal X\to\mathbb R$ such that
\begin{itemize}
\item[(i)] $\int s(x)\,\nu(dx)=0$;
\item[(ii)] $\|s\|_\infty\le 1$;
\item[(iii)] $s\not\equiv 0$ $\nu$-almost surely.
\end{itemize}
Consequently, the signed measure
\begin{equation}\label{eq:signed-measure-S}
S(H):=\int_H s(x)\,\nu(dx)
\end{equation}
has the following properties:
\begin{itemize}
\item[(a)] $S(\mathcal X)=0$;
\item[(b)] $S\neq 0$;
\item[(c)] $|S|(H)\le \nu(H)$ for every measurable set $H\subseteq\mathcal X$.
\end{itemize}
\end{proposition}

\begin{proof}
Choose a measurable set $B\subseteq\mathcal X$ with $0<\nu(B)<1$ and define
\[
s(x):=\mathbf 1_B(x)-\nu(B),
\]
where $\mathbf 1_B(x)$ denotes the indicator of $B$, equal to $1$ on $B$ and $0$ outside $B$. We verify the properties in order. For (i),
\[
\int s\,d\nu
=
\int \mathbf 1_B\,d\nu-\nu(B)\int 1\,d\nu
=
\nu(B)-\nu(B)\cdot 1
=
0.
\]
Property (ii) holds because $\mathbf 1_B(x)\in\{0,1\}$, so $s(x)$ is either $1-\nu(B)$ or $-\nu(B)$, and both have absolute value at most one. Property (iii) holds because $0<\nu(B)<1$, so $s$ is positive on $B$ and negative on $B^c$, both of which have positive $\nu$-measure.

For $S$ as defined in Eq.~\eqref{eq:signed-measure-S}, property (a) is exactly (i) evaluated on $H=\mathcal X$. For property (b), consider the measurable set $H_+:=\{x:s(x)>0\}$. Because $s$ is not zero $\nu$-almost surely (by (iii)), either $\nu(H_+)>0$ or $\nu(\{x:s(x)<0\})>~0$. In the first case, $S(H_+)=\int_{H_+} s\,d\nu>0$; in the second case, taking $H_-:=\{x:s(x)<0\}$ gives $S(H_-)<0$. Hence $S\neq0$. 

Finally, for (c),
\[
|S|(H)=\int_H |s(x)|\,\nu(dx)\le \int_H 1\,\nu(dx)=\nu(H),
\]
where we use $\|s\|_\infty\le 1$ from (ii) and $|s(x)|\le \|s\|_\infty$.
\end{proof}

This proposition is the seed of the perturbation argument. The later kernel construction will add the perturbation $S$ to $Q_0$ pointwise at each intervention pair $(z_\rho,z_\sigma)$. For $Q_0$ to remain a probability kernel, the perturbation must change some events but contribute zero net mass overall; that is exactly what the zero-total-mass property of $S$ just proven in this proposition guarantees. We next look at the perturbation effect on the marginals.

\begin{proposition}[Perturbation family with fixed marginals]\label{prop:kernel-family-new}
The event-level form of Eq.~\eqref{eq:perturbation} is
\begin{equation}\label{eq:perturbation-H}
Q_t(H\mid z_\rho,z_\sigma)
:=
Q_0(H\mid z_\rho,z_\sigma)+t\,u(z_\rho)v(z_\sigma)S(H).
\end{equation}
Then for every
\[
|t|<\delta:=\frac{\varepsilon}{\|u\|_\infty\|v\|_\infty},
\]
$Q_t$ is a probability kernel and belongs to the kernel-level compatibility set $\mathcal Q_\tau(\tilde\chi^\rho,\tilde\chi^\sigma)$ of kernels compatible with the two private reports.
\end{proposition}

\begin{proof}
Fix $(z_\rho,z_\sigma)$. To show that $Q_t(\cdot\mid z_\rho,z_\sigma)$ is a probability measure, we check normalisation and nonnegativity.

For normalisation, evaluate the definition of $Q_t$ on the whole outcome space:
\[
Q_t(\mathcal X\mid z_\rho,z_\sigma)
=
Q_0(\mathcal X\mid z_\rho,z_\sigma)+t\,u(z_\rho)v(z_\sigma)S(\mathcal X).
\]
Because $Q_0$ is an admissible kernel, evaluating it on the whole outcome space gives $Q_0(\mathcal X\mid z_\rho,z_\sigma)=1$. Because Proposition~\ref{prop:signed-perturbation-new}(a) gives $S(\mathcal X)=0$, the perturbation contributes no net mass. Hence
\[
Q_t(\mathcal X\mid z_\rho,z_\sigma)=1.
\]

For nonnegativity, consider two cases. If 
$(z_\rho, z_\sigma) \notin A_\rho \times A_\sigma$, then by the 
support condition in \ref{ass:A2}, $u(z_\rho) v(z_\sigma) = 0$, so 
$Q_t(\cdot \mid z_\rho, z_\sigma) = Q_0(\cdot \mid z_\rho, z_\sigma)$ 
is a probability measure trivially. If 
$(z_\rho, z_\sigma) \in A_\rho \times A_\sigma$, the local 
minorisation condition \ref{ass:A3} applies through its lower-bound inequality. Using the elementary bounds 
$|u(z_\rho)| \leq \|u\|_\infty$, $|v(z_\sigma)| \leq \|v\|_\infty$, 
together with the defining total-variation inequality 
$S(H) \geq -|S|(H)$ for signed measures:
\begin{equation*}
Q_t(H \mid z_\rho, z_\sigma) 
\geq Q_0(H \mid z_\rho, z_\sigma) 
- |t| \, \|u\|_\infty \|v\|_\infty \, |S|(H),
\end{equation*}
for every measurable set $H\subseteq\mathcal X$. By 
assumption \ref{ass:A3}, $Q_0(H \mid z_\rho, z_\sigma) \geq \varepsilon \nu(H)$ 
on this rectangle, and by Proposition~\ref{prop:signed-perturbation-new}(c), 
$|S|(H) \leq \nu(H)$. Substituting these two bounds gives
\begin{equation*}
Q_t(H \mid z_\rho, z_\sigma) 
\geq \big(\varepsilon - |t| \, \|u\|_\infty \|v\|_\infty\big) \, \nu(H) 
\geq 0
\end{equation*}
whenever $|t| < \delta = \varepsilon / (\|u\|_\infty \|v\|_\infty)$, 
because then $\varepsilon - |t| \, \|u\|_\infty \|v\|_\infty > 0$ and 
$\nu(H) \geq 0$. So $Q_t(\cdot \mid z_\rho, z_\sigma)$ is a probability 
measure for every $(z_\rho, z_\sigma)$.

Integrating the perturbed kernel of Eq.~\eqref{eq:perturbation} over $P_{Z_\sigma}$ gives
\[
\begin{aligned}
&\int Q_t(\cdot\mid z_\rho,z_\sigma)\,P_{Z_\sigma}(dz_\sigma)\\
&\quad =
\int Q_0(\cdot\mid z_\rho,z_\sigma)\,P_{Z_\sigma}(dz_\sigma)\\
&\qquad +
t\,u(z_\rho)\Bigl(\int v(z_\sigma)\,P_{Z_\sigma}(dz_\sigma)\Bigr)S.
\end{aligned}
\]
The second term vanishes by the centring condition in \ref{ass:A2}, so $\pi^\rho(Q_t)=\pi^\rho(Q_0)=\tilde\chi^\rho$. The same calculation integrating over $P_{Z_\rho}$ gives $\pi^\sigma(Q_t)=\pi^\sigma(Q_0)=\tilde\chi^\sigma$. Hence $Q_t\in\mathcal Q_\tau(\tilde\chi^\rho,\tilde\chi^\sigma)$.
\end{proof}

This proposition proves non-uniqueness of the Markov kernel characterising the backbone. Proposition~\ref{prop:signed-perturbation-new} only gives a direction on the outcome space; here that direction becomes a whole interval of admissible backbone kernels that are observationally indistinguishable from the agents' point of view because the private report operators integrate one centred factor away. Finally, we show that these kernels are not interventionally consistent as required by RCK.

\begin{proposition}[Interventional non-equivalence]\label{prop:interventional-nonequivalence-new}
For $t\neq s$, there exists a pair $(z_\rho^\star,z_\sigma^\star)$ such that
\[
Q_t(\cdot\mid z_\rho^\star,z_\sigma^\star)\neq Q_s(\cdot\mid z_\rho^\star,z_\sigma^\star),
\]
and therefore the two admissible backbones induce different joint interventional distributions under the joint intervention
\(
\mathrm{do}(Z_\rho=z_\rho^\star,Z_\sigma=z_\sigma^\star).
\)
\end{proposition}

\begin{proof}
By assumption \ref{ass:A2}, $u$ is not $P_{Z_\rho}$-almost surely zero and vanishes outside $A_\rho$, while $v$ is not $P_{Z_\sigma}$-almost surely zero and vanishes outside $A_\sigma$. Therefore 
there exist $z_\rho^\star \in A_\rho$ and $z_\sigma^\star \in A_\sigma$ 
such that $u(z_\rho^\star) \neq 0$ and $v(z_\sigma^\star) \neq 0$, and 
hence
\begin{equation*}
u(z_\rho^\star) v(z_\sigma^\star) \neq 0.
\end{equation*}
By Proposition~\ref{prop:signed-perturbation-new}(b), the signed measure $S$ is nonzero, so there exists a measurable set $H^\star\subseteq\mathcal X$ with $S(H^\star)\neq 0$. Then
\[
\begin{aligned}
&Q_t(H^\star\mid z_\rho^\star,z_\sigma^\star)
    -Q_s(H^\star\mid z_\rho^\star,z_\sigma^\star)\\
&\quad =(t-s)\,u(z_\rho^\star)v(z_\sigma^\star)S(H^\star)\neq 0.
\end{aligned}
\]
So $Q_t$ and $Q_s$ are distinct kernels, and they differ on the interventional distribution associated with the joint intervention $(z_\rho^\star,z_\sigma^\star)$.
\end{proof}

Given these components, we can now prove the core non-identifiability result of Theorem~\ref{thm:main}.

\begin{proof}[Proof of Theorem~\ref{thm:main}]
Proposition~\ref{prop:signed-perturbation-new} provides a nontrivial signed measure $S$ on the outcome space with total mass zero. This gives a direction in which one can perturb an admissible kernel without breaking the normalisation condition. Proposition~\ref{prop:kernel-family-new} shows that, after modulating this direction by the centred functions $u$ and $v$, one obtains a whole one-parameter family of admissible kernels 
\[\{Q_t:|t|<\delta\}\subseteq\mathcal Q_\tau(\tilde\chi^\rho,\tilde\chi^\sigma)\] 
that all reproduce the same observed private reports. Hence the private reports do not identify a unique backbone candidate. Proposition~\ref{prop:interventional-nonequivalence-new} then shows that distinct members of this family disagree on some joint intervention. Therefore the ambiguity is causally meaningful rather than merely representational. This proves, under assumptions \ref{ass:A1}--\ref{ass:A3}, that a generic backbone Markov kernel is not identifiable when private reports collide on a shared outcome.
\end{proof}

\subsection{RCK Interpretation of Theorem~\ref{thm:main}}\label{app:rck-proof}
\begin{proposition}\label{prop:rck-obstruction-main}
In the private-cause kernel representation of the edge stalk, the private reports do not identify a unique common edge value on the RCK edge stalk: under Assumptions~\ref{ass:A1}--\ref{ass:A3}, $|\mathcal B_\tau(\tilde\chi^\rho,\tilde\chi^\sigma)|>1$.
\end{proposition}
\begin{proof}
By Propositions~\ref{prop:kernel-family-new} and~\ref{prop:interventional-nonequivalence-new}, there exist distinct kernels $Q_t,Q_s\in\mathcal Q_\tau(\tilde\chi^\rho,\tilde\chi^\sigma)$ with $t\neq s$ that reproduce the same private reports $\tilde\chi^\rho$ and $\tilde\chi^\sigma$ and disagree on some joint intervention. In the private-cause representation, these kernels are representatives of distinct edge-level candidates $\chi_t^\tau,\chi_s^\tau\in CK(M^\tau)$. Since membership in $\mathcal Q_\tau(\tilde\chi^\rho,\tilde\chi^\sigma)$ means that each kernel projects to $\tilde\chi^\rho$ under $\pi^\rho$ and to $\tilde\chi^\sigma$ under $\pi^\sigma$, both edge-level candidates lie in
\(
\mathcal B_\tau(\tilde\chi^\rho,\tilde\chi^\sigma).
\)
Because $Q_t$ and $Q_s$ differ on some joint intervention, the corresponding edge-level candidates are distinct. Hence
\[
\left|\mathcal B_\tau(\tilde\chi^\rho,\tilde\chi^\sigma)\right|>1.
\]
\end{proof}

\subsection{Proof of Proposition~\ref{prop:separable}}\label{app:separable-proof}

\begin{proposition*}[Proposition~\ref{prop:separable}, restated]
If the shared outcome $X$ satisfies Eq.~\eqref{eq:separable} and $\E[\varepsilon]<\infty$, then the joint interventional mean response is
\[
\begin{aligned}
m(z_\rho,z_\sigma)
&:=
\E[X\mid \mathrm{do}(Z_\rho=z_\rho,Z_\sigma=z_\sigma)]\\
&= 
f_\rho(z_\rho)+f_\sigma(z_\sigma)+\E[\varepsilon].
\end{aligned}
\]
In particular, the response contains no mixed term in $(z_\rho,z_\sigma)$.
\end{proposition*}

\begin{proof}
Under the intervention $\mathrm{do}(Z_\rho=z_\rho,Z_\sigma=z_\sigma)$, the structural equation becomes
\[
X=f_\rho(z_\rho)+f_\sigma(z_\sigma)+\varepsilon.
\]
Taking expectations yields the interventional mean response in Proposition~\ref{prop:separable}. The perturbation family in Eq.~\eqref{eq:perturbation} varies admissible backbones through the jointly varying factor $u(z_\rho)v(z_\sigma)$ while preserving both one-cause marginals. An additive response surface leaves no room for such a mixed term, so that perturbation direction is excluded at the level of the induced backbone response surface.
\end{proof}
\vspace{-0.4cm}
\subsection{Observational summaries do not identify the additive response}\label{app:obs-proof}
\vspace{-0.2cm}
\begin{proposition*}[Observational residual summaries are insufficient]
Even under the additive model in Eq.~\eqref{eq:separable}, communicated observational residual summaries such as residual variances need not identify the additive backbone mean response $m(z_\rho,z_\sigma)$.
\end{proposition*}

\begin{proof}
It suffices to give a counterexample. Let
\[
Z_\rho,Z_\sigma\stackrel{\mathrm{iid}}{\sim}\mathcal N(0,1),
\qquad
f_\sigma(z)=z,
\qquad
\varepsilon\sim \mathcal N(0,\eta^2),
\]
with $Z_\rho$, $Z_\sigma$, and $\varepsilon$ mutually independent. Compare the two additive models that share the same nontrivial $\sigma$-channel $f_\sigma$ but differ in the $\rho$-channel:
\[
f_\rho^{(1)}(z)=z,
\qquad
f_\rho^{(2)}(z)=\frac{z^2-1}{\sqrt 2}.
\]
Both satisfy
\(
\Var\bigl(f_\rho^{(1)}(Z_\rho)\bigr)
=
\Var\bigl(f_\rho^{(2)}(Z_\rho)\bigr)
=
1.
\)

In both models,
\[
\begin{aligned}
R_\rho&=X-f_\rho(Z_\rho)
    =f_\sigma(Z_\sigma)+\varepsilon=Z_\sigma+\varepsilon,\\
\Var(R_\rho)&=1+\eta^2,
\end{aligned}
\]
and
\[
\begin{aligned}
R_\sigma&=X-f_\sigma(Z_\sigma)
=f_\rho(Z_\rho)+\varepsilon,
\\
\Var(R_\sigma)&=1+\eta^2.
\end{aligned}
\]
So the observational residual summaries coincide, but the additive backbone response surfaces are different:
\[
\begin{aligned}
m_1(z_\rho,z_\sigma)&=z_\rho+z_\sigma,\\
m_2(z_\rho,z_\sigma)&=\frac{z_\rho^2-1}{\sqrt 2}+z_\sigma.
\end{aligned}
\]
Hence, observational residual summaries do not identify $m$ in general.
\end{proof}
\vspace{-0.6cm}
\section{Illustrative Witnesses of Theorem~\ref{thm:main}}\label{app:witnesses}
\vspace{-0.2cm}
This appendix records two low-dimensional witnesses and one continuous non-Gaussian witness to illustrate the general non-identifiability result in Theorem~\ref{thm:main}. They are concrete illustrations of the same missing interaction direction that the kernel-level proof formalises.
They are not claimed as standalone new marginal-pathology examples: the binary case echoes familiar contingency-table and marginal-compatibility phenomena \citep{agresti2013categorical,nelsen2006copulas}, while the Gaussian case echoes linear-Gaussian non-identifiability phenomena \citep{peters2017elements}. Their role is to show how the same missing interaction direction appears inside the RCK backbone-identification problem, where the fixed objects are communicated one-cause causal marginals rather than a single observational law.
\vspace{-0.2cm}
\subsection[Binary 2x2 witness]{The binary $2\times 2$ witness}\label{app:witness-binary}
\vspace{-0.2cm}
Let $p_{ij}:=P(X=1\mid Z_\rho=i,Z_\sigma=j)$ for $i,j\in\{0,1\}$, and suppose the unobserved private cause is averaged with weight $\frac12$. Write
\[
a_i:=\tilde\chi^\rho(\{1\}\mid i),
\qquad
b_j:=\tilde\chi^\sigma(\{1\}\mid j).
\]
Then the reported one-cause kernels impose the linear constraints
\[
\frac12 p_{00}+\frac12 p_{01}=a_0,\qquad
\frac12 p_{10}+\frac12 p_{11}=a_1,
\]
\[
\frac12 p_{00}+\frac12 p_{10}=b_0,\qquad
\frac12 p_{01}+\frac12 p_{11}=b_1.
\]
This system has rank three, not four, because the sum of the first two equations equals the sum of the second two. Its null space is spanned by
\(
(1,-1,-1,1).\\[-0.2cm]
\)

So if $p_0=(p_{00}^{(0)},p_{01}^{(0)},p_{10}^{(0)},p_{11}^{(0)})$ is one admissible table, then every table of the form
\[
p(t)=p_0+t(1,-1,-1,1)
\]
has the same row and column averages for all sufficiently small $|t|$. The unresolved coordinate is the interaction term.
\vspace{-0.2cm}
\subsection{The Gaussian covariance witness}\label{app:witness-gaussian}
\vspace{-0.2cm}
For a moment-based witness, let the shared outcome be vector-valued, $X=(X_1,X_2)$, and let $(Z_\rho,Z_\sigma)\in\{0,1\}^2$. In each joint intervention cell, define a Gaussian response
\vspace{-0.1cm}
\[
X\mid \mathrm{do}(Z_\rho=z_\rho,Z_\sigma=z_\sigma)
\sim
\mathcal N\!\bigl(0,\Sigma_t(z_\rho,z_\sigma)\bigr),\\[-0.1cm]
\]
with covariance matrix
\vspace{-0.1cm}
\[
\Sigma_t(z_\rho,z_\sigma)
:=
\begin{pmatrix}
1 & t(-1)^{z_\rho+z_\sigma}\\
t(-1)^{z_\rho+z_\sigma} & 1
\end{pmatrix},
\qquad |t|<1.
\]
The cell-wise cross-covariance is therefore
\[
\begin{aligned}
\sigma^{(t)}_{12}(z_\rho,z_\sigma)
&:=
\Cov_t\!\left(X_1,X_2\mid
    \mathrm{do}(Z_\rho=z_\rho,Z_\sigma=z_\sigma)\right)\\
&=
t(-1)^{z_\rho+z_\sigma}.
\end{aligned}
\]
If an agent averages over the other private cause with equal weights, then the reported one-cause conditional distribution is a mixture of two Gaussians: $\frac12\mathcal{N}(0,\Sigma_t(z_\rho,0)) + \frac12\mathcal{N}(0,\Sigma_t(z_\rho,1))$, not itself Gaussian. However, the covariance matrix of this mixture is the equal-weighted average of the cell-wise covariances, so the reported row and column second-moment summaries are
\[
\begin{aligned}
\frac12 \sigma^{(t)}_{12}(z_\rho,0)
    +\frac12 \sigma^{(t)}_{12}(z_\rho,1)&=0,\\
\frac12 \sigma^{(t)}_{12}(0,z_\sigma)
    +\frac12 \sigma^{(t)}_{12}(1,z_\sigma)&=0.
\end{aligned}
\]
These summaries vanish because the cross-covariance entries $\sigma^{(t)}_{12}(z_\rho,z_\sigma)=t(-1)^{z_\rho+z_\sigma}$ alternate in sign across the intervention cells: for fixed $z_\rho$, the values $t(-1)^{z_\rho}$ and $t(-1)^{z_\rho+1}$ are negatives of each other, so they cancel under equal weighting.
Hence, every value of $t$ induces the same \emph{communicated} cross-covariance summaries, even though the joint interventional laws differ cell by cell. In particular,
\vspace{-0.2cm}
\[
\begin{aligned}
t=0&\implies X_1\indep X_2,\\
t\neq 0&\implies X_1\not\!\indep X_2,
\end{aligned}
\]
within each intervention cell. Here, by ``communicated cross-covariance summaries'' we mean the second-moment level: the mixture distributions themselves differ from cell to cell, but their marginal covariance matrices remain identical across all values of $t$. This witness shows that even a single low-dimensional functional of the joint intervention law can remain free under the same local summaries.
\vspace{-0.2cm}
\subsection{A continuous, non-Gaussian witness}\label{sec:continuous-witness}
\vspace{-0.2cm}
The witnesses in Sections~\ref{app:witness-binary} and 
\ref{app:witness-gaussian} are finite-dimensional. The kernel-level 
statement of Theorem~\ref{thm:main} applies more broadly, and the 
construction in Proposition~\ref{prop:signed-perturbation-new} produces 
an explicit perturbation on any nontrivial measurable outcome space. 
We illustrate the Theorem's applicability with a continuous outcome and non-Gaussian base kernel.

Let $\mathcal{X} = \mathbb{R}$ with Borel $\sigma$-algebra, and let 
$Z_\rho, Z_\sigma$ each take values in $[-1, 1]$, with symmetric marginal 
laws $P_{Z_\rho}, P_{Z_\sigma}$ admitting densities bounded away from zero. Take the base kernel
\vspace{-0.2cm}
\begin{equation*}
Q_0(\cdot \mid z_\rho, z_\sigma) 
= \mathrm{Laplace}\big(\mu(z_\rho, z_\sigma),\, b\big),\\[-0.1cm]
\end{equation*}
where $\mu(z_\rho, z_\sigma) = z_\rho + z_\sigma$ is the additive mean 
and $b > 0$ is fixed. The Laplace base satisfies the local overlap condition in Assumption~\ref{ass:A3}: on the intervention rectangle $A_\rho\times A_\sigma=[-1,1]^2$, the mean $\mu(z_\rho,z_\sigma)$ remains in $[-2,2]$, so the Laplace density is bounded below on any bounded outcome interval. Choose $\nu$ as the uniform law on $[-M,M]$ for some $M>1$ and let $c>0$ be the corresponding lower bound, so that
\vspace{-0.1cm}
\[
Q_0(H\mid z_\rho,z_\sigma)\ge c\,\nu(H)\\[-0.1cm]
\]
for all $(z_\rho,z_\sigma)\in[-1,1]^2$ and measurable $H\subseteq\mathcal X$.

Choose the perturbation seed 
$s(x) = \mathbf{1}_{[0,1]}(x) - \mathbf{1}_{[-1,0]}(x)$, which is 
bounded, $\nu$-centred, and not $\nu$-almost-surely zero. The signed 
measure $S(H) = \int_H s \, d\nu$, in the form of Eq.~\eqref{eq:signed-measure-S}, has $S(\mathcal{X}) = 0$ as in Proposition~\ref{prop:signed-perturbation-new}(a) and is 
supported on $[-1, 1]$. With $u(z_\rho) = z_\rho$ 
and $v(z_\sigma) = z_\sigma$, the symmetry of the marginal laws verifies the centring part of Assumption~\ref{ass:A2}. The perturbation family, i.e. Eq.~\eqref{eq:perturbation-H} in this witness,
\begin{equation*}
Q_t(\cdot \mid z_\rho, z_\sigma) 
= Q_0(\cdot \mid z_\rho, z_\sigma) 
+ t \, z_\rho \, z_\sigma \, S(\cdot)
\end{equation*}
consists of valid probability kernels for all $|t| < c$. Each 
$Q_t$ produces the same single-cause marginal kernels as $Q_0$, but 
the joint interventional distribution at 
$(z_\rho^\star, z_\sigma^\star) = (1, 1)$ shifts mass from the negative 
half-line to the positive half-line in proportion to $t$. The moments 
of the single-cause reports coincide across the family, while the 
joint interventional law differs at any intervention pair with 
$z_\rho \cdot z_\sigma \neq 0$.

This witness illustrates that the obstruction is not a 
finite-dimensional or Gaussian artefact: it survives on continuous 
outcomes, with non-Gaussian base kernels, and extends to heavy-tailed 
alternatives obtained by replacing the Laplace base with any kernel 
admitting a local overlap bound on a positive-measure rectangle.

\section{Additive special cases under causal-response communication}\label{app:positive-special-cases}

Throughout this appendix section, assume the additive-separable model in Eq.~\eqref{eq:separable}. The agents communicate causally identified response functions $f_\rho$ and $f_\sigma$, as required in Section~3 of the main text. With disturbance law $P_\varepsilon$, the corresponding additive edge kernel is
\begin{equation}\label{eq:additive_backbone_kernel}
Q(H\mid z_\rho,z_\sigma)
:=
P_\varepsilon\!\left(\{e:m(z_\rho,z_\sigma)+e\in H\}\right).
\end{equation}
Thus communicating $f_\rho$ and $f_\sigma$ identifies the intervention-dependent part $m$; the full kernel is fixed once $P_\varepsilon$ is fixed, identified, or communicated. 

We record two special cases to illustrate the positive identifiability consequences of communicating the additive response functions.
The first special case records when communicated slopes are enough to identify the additive response surface; the second records when the disturbance variance is also recoverable from an observed marginal variance under additional independence assumptions.
\vspace{-0.2cm}
\paragraph{Linear additive responses.}
Suppose further that
\[
f_\rho(z_\rho)=\alpha z_\rho,
\qquad
f_\sigma(z_\sigma)=\beta z_\sigma.
\]
Then the identified slopes $(\alpha,\beta)$ are sufficient to identify
\[
m(z_\rho,z_\sigma)=\alpha z_\rho+\beta z_\sigma.
\]
Within the linear additive class, they are also necessary because they fully determine the response functions; communicating those functions then fixes the additive backbone response surface above.
\vspace{-0.2cm}
\paragraph{Residual variance under independence.}
Suppose in addition that $Z_\rho\indep Z_\sigma$, $\varepsilon\indep(Z_\rho,Z_\sigma)$, and the marginal variance $\Var(X)$ is observed. Then
\[
\Var(\varepsilon)
=
\Var(X)-\Var\!\bigl(f_\rho(Z_\rho)\bigr)-\Var\!\bigl(f_\sigma(Z_\sigma)\bigr).
\]
Because $Z_\rho\indep Z_\sigma$ and $\varepsilon\indep(Z_\rho,Z_\sigma)$, the three terms $f_\rho(Z_\rho)$, $f_\sigma(Z_\sigma)$, and $\varepsilon$ are pairwise uncorrelated. Therefore
\[
\Var(X)
=
\Var\!\bigl(f_\rho(Z_\rho)\bigr)+\Var\!\bigl(f_\sigma(Z_\sigma)\bigr)+\Var(\varepsilon),
\]
identifying the disturbance variance under these additional moment and independence assumptions.

\end{document}